\documentclass[a4paper]{article}

\usepackage{graphicx}
\usepackage{fancyhdr}
\usepackage{fancybox}
\usepackage{natbib}
\usepackage{hyperref}
\usepackage{xcolor}
\usepackage{amssymb,amsfonts,latexsym,amsmath}
\usepackage{natbib}
\usepackage{psfrag}
\usepackage{auto-pst-pdf}
\newtheorem{definition}{Definition}
\newtheorem{theorem}{Theorem}
\newtheorem{corollary}{Corollary}
\newtheorem{lemma}{Lemma}
\newtheorem{proposition}{Proposition}
\newtheorem{example}{Example}

\newenvironment{proof}[1][Proof]{\begin{trivlist}
\item[\hskip \labelsep {\bfseries #1}]}{\hfill$\Box$\end{trivlist}}
\newenvironment{proof2}[1][Proof]{\begin{trivlist}
\item[\hskip \labelsep {\bfseries #1}]}{\end{trivlist}}
\newenvironment{proof3}[1][]{\begin{trivlist}
\item[\hskip \labelsep {\bfseries #1}]}{\hfill$\Box$\end{trivlist}}
\usepackage{parskip}

\hypersetup{
    pdftitle={Robustness and Generalization for Metric Learning},    
    pdfauthor={Aurelien Bellet and Amaury Habrard},     
}

\begin{document} 

\lhead{Aur\'elien Bellet, Amaury Habrard} 
\rhead{Robustness and Generalization for Metric Learning}
\cfoot{\thepage} 

\renewcommand{\headrulewidth}{0.4pt}  

\title{Robustness and Generalization for Metric Learning}

\author{
Aur\'elien Bellet\thanks{Department of Computer Science, University of Southern California. Email: \texttt{bellet@usc.edu}.}~\thanks{Most of the work in this paper was carried out while the author was affiliated with Laboratoire Hubert Curien UMR CNRS 5516, Universit\'e Jean Monnet, 42000 Saint-Etienne, France.}
\and Amaury Habrard\thanks{Laboratoire Hubert Curien UMR CNRS 5516, Universit\'e Jean Monnet, 42000 Saint-Etienne, France. Email: \texttt{amaury.habrard@univ-st-etienne.fr}}
}

\date{}

\maketitle

\begin{abstract}
Metric learning has attracted a lot of interest over the last decade, but the generalization ability of such methods has not been thoroughly studied. In this paper, we introduce an adaptation of the notion of algorithmic robustness (previously introduced by Xu and Mannor) that can be used to derive generalization bounds for metric learning. We further show that a weak notion of robustness is in fact a 
necessary and sufficient condition for a metric learning algorithm to generalize. To illustrate the applicability of the proposed framework, we derive generalization results for a large family of existing metric learning algorithms, including some sparse formulations that are not covered by previous results.
\end{abstract}

\section{Introduction}

Metric learning consists in automatically adjusting a distance or similarity function using training examples. The resulting metric is tailored to the problem of interest and can lead to dramatic improvement in classification, clustering or ranking performance. For this reason, metric learning has attracted a lot of interest for the past decade (see \cite{Bellet2013,Kulis2012} for recent surveys). 
Existing approaches rely on the principle that pairs of examples with the same (resp. different) labels should be close to each other (resp. far away) under a good metric. Learning thus generally consists in finding the best parameters of the metric function given a set of labeled pairs.\footnote{Some methods use triplets $(x,y,z)$ such that $x$ should be closer to $y$ than to $z$, where $x$ and $y$ share the same label, but not $z$.}
Many methods focus on learning a Mahalanobis distance, which is parameterized by a positive semi-definite (PSD) matrix and can be seen as finding a linear projection of the data to a space where the Euclidean distance performs well on the training pairs (see for instance \cite{Xing2002,Schultz2003,Davis2007,Jain2008,Weinberger2009,Ying2009,McFee2010}). More flexible metrics have also been considered, such as similarity functions without PSD constraint   \cite{Chechik2009,Qamar2010,Bellet2012a}. The resulting distance or similarity is used to improve the performance of a metric-based algorithm such as $k$-nearest neighbors \cite{Davis2007,Weinberger2009}, linear separators \cite{Bellet2012a,Guo2014}, $K$-Means clustering \cite{Xing2002} or ranking \cite{McFee2010}.


Despite the practical success of metric learning, little work has gone into a formal analysis of the generalization ability of the resulting metrics on unseen data. The main reason for this lack of results is that metric learning violates the common assumption of independent and identically distributed (IID) data. Indeed, the training pairs are generally given by an expert and/or extracted from a
sample of individual instances, by considering all possible pairs or only a subset based for instance on the nearest
or farthest neighbors of each example, some criterion of diversity \cite{Kar2011} or a random sample. 
Online learning algorithms \cite{Shalev-Shwartz2004,Jain2008,Chechik2009} can still offer some guarantees in this setting, but only in the form of regret bounds assessing the deviation between the cumulative loss suffered by the online algorithm and the loss induced by the best hypothesis that can be chosen in hindsight. These may be converted into proper generalization bounds under restrictive assumptions \cite{Wang2012c}.
Apart from these results on online metric learning, very few papers have looked at the generalization ability of batch methods. The approach  of Bian and Tao \cite{Bian2011,Bian2012} uses a statistical analysis to give 
generalization guarantees for loss minimization approaches, but their results rely on restrictive assumptions on the distribution of the examples and do not take into account any regularization on the metric. Jin et al. \cite{Jin2009} adapted the framework of 
uniform stability \cite{Bousquet2002} to regularized metric
learning. However, their approach is based on a Frobenius norm regularizer and cannot be applied to other type of regularization, in particular sparsity-inducing norms \cite{Xu2012} that are used in many recent metric learning approaches \cite{Rosales2006,Ying2009,Qi2009,McFee2010}. Independently and in parallel to our work, Cao et al. \cite{Cao2012a} proposed a framework based on Rademacher analysis, which is general but rather complex and limited to pair constraints.

In this paper, we propose to study the generalization ability of metric learning algorithms according to a notion of {\it algorithmic robustness}. 
This framework, introduced by Xu et al. \cite{Xu2010,Xu2012a}, allows one to derive 
generalization bounds when the variation in the loss associated with two ``close'' training and testing examples is bounded. 
The notion of closeness relies on a partition of the input space into
different regions such that two examples in the same region are considered close. 
Robustness has been successfully used to derive generalization bounds in the classic supervised learning setting, with results for SVM, LASSO, etc. 
We propose here to adapt algorithmic robustness to metric learning. 
We show that, in this context, the problem of non-IIDness of the training pairs/triplets can be worked around by simply assuming that they are built from an IID sample of labeled examples. 
Moreover, following \cite{Xu2012a}, we provide a notion of weak robustness that is necessary and sufficient for metric learning algorithms to generalize well, confirming that robustness is a fundamental property. 
We illustrate the applicability of the proposed framework by deriving generalization bounds, using very few approach-specific arguments, for a family of problems that is larger than what is considered in previous work \cite{Jin2009,Bian2011,Bian2012,Cao2012a}. In particular, results apply to a vast choice of regularizers,  without any assumption on the distribution of the examples and using a simple proof technique.

The rest of the paper is organized as follows. We introduce some preliminaries and notations in Section~\ref{prelim}. Our notion of algorithmic robustness for metric learning is presented in Section~\ref{robustsec}. The necessity and sufficiency  of weak robustness is shown in Section~\ref{nessec}. Section~\ref{exsec} illustrates the wide applicability of our framework by deriving bounds for existing metric learning formulations. Section~\ref{disc} discusses the merits and limitations of the proposed analysis compared to related work, and we conclude in Section~\ref{conclu}.

\section{Preliminaries}
\label{prelim}

\subsection{Notations}

Let $X$ be the instance space, $Y$
be a finite label set and let $\mathcal{Z}=X\times Y$. In the following, $z=(x,y)\in\mathcal{Z}$ means $x\in X$ and $y\in Y$. Let $\mu$ be an unknown 
probability distribution over $\mathcal{Z}$.
 We assume  
that $X$ is a compact convex metric space w.r.t. a norm $\|\cdot\|$ such that $X\subset\mathbb{R}^d$, thus there exists a constant $R$ such that $\forall x\in X$, $\|x\|\leq R$.  
A similarity or distance function is a pairwise function $f:X\times
X\rightarrow \mathbb{R}$.  
In the following, we use the generic term {\it metric} to refer to either a similarity or a distance function.    
We denote by $\mathbf{s}$  a labeled training
sample consisting of $n$ training instances $(s_1,\ldots,s_n)$ drawn
IID from $\mu$.
The sample of all possible pairs built from $\mathbf{s}$ is denoted by $p_{\mathbf{s}}$ such that  $p_{\mathbf{s}}=\{(s_1,s_1),\ldots,(s_1,s_n),\ldots,(s_n,s_n)\}$.
A metric learning algorithm $\mathcal{A}$ takes as input a finite set of pairs from
$(\mathcal{Z}\times\mathcal{Z})^n$ and outputs a metric.
We denote by $\mathcal{A}_{p_{\mathbf{s}}}$ the metric learned by an algorithm $\mathcal{A}$ from a sample $p_{\mathbf{s}}$ of pairs. 
For any pair of labeled examples $(z,z')$ and any metric $f$, we associate a loss function $l(f,z,z')$ which depends on the examples and their labels. This loss is assumed to be nonnegative and uniformly bounded by a constant $B$.  
We define the generalization loss (or true loss) over $\mu$ as
$$\mathcal{L}(f)=\mathbb{E}_{z,z'\sim \mu}l(f,z,z'),$$ 
and the empirical loss over the sample $p_{\mathbf{s}}$ as
$$l_{emp}(f)=\frac{1}{n^2}\sum_{i=1}^n\sum_{j=1}^nl(f,s_i,s_j)=\frac{1}{n^2}\sum_{(s_i,s_j)\in p_{\mathbf{s}}}l(f,s_i,s_j).$$
We are interested in bounding the deviation between $l_{emp}(f)$ and $\mathcal{L}(f)$.

\subsection{Algorithmic Robustness in Classic Supervised Learning}

The notion of algorithmic robustness, introduced by Xu and Mannor \cite{Xu2010,Xu2012a} in the context of classic supervised learning, is based on the deviation between the loss associated with two training and testing instances that are ``close''.  Formally, an algorithm is said $(K,\epsilon(\mathbf{s}))$-robust if there exists a partition of the  space $\mathcal{Z}=X\times Y$ into $K$ disjoint subsets such that for every training and testing instances belonging to the same region of the partition, the variation in their associated loss is   bounded by a term $\epsilon(\mathbf{s})$. From this definition, the authors have proved a bound for the difference between the empirical loss and the true loss that has the form
\begin{equation}
\label{robbound}
\epsilon(\mathbf{s})+B\sqrt{\frac{2K \ln 2 + 2\ln 1/\delta}{n}},
\end{equation}
with probability $1-\delta$.  
This bound depends on $K$ and $\epsilon(\mathbf{s})$. The latter should tend to zero as $K$ increases to ensure that \eqref{robbound} also goes to zero when $n\rightarrow \infty$.\footnote{This point will be made clear by the examples provided in Section~\ref{exsec}.}
When considering metric spaces, the partition of $\mathcal{Z}$  can be obtained by  the notion of covering number \cite{Kolmogorov1961}.
\begin{definition}
For a metric space $(X,\rho)$, and $T\subset X$, we say that
$\hat{T}\subset T$ is a $\gamma$-cover of $T$, if $\forall t\in T$,
$\exists \hat{t}\in \hat{T}$ such that $\rho(t,t')\leq \gamma$. The
$\gamma$-covering number of $T$ is
$$
{\mathcal N}(\gamma,T,\rho)=\min\{|\hat{T}|: \hat{T} \mbox{\ is a\ }
\gamma-\mbox{\ cover of\ }T\}.
$$
\end{definition}
When $X$ is a compact convex space, for any $\gamma>0$, the quantity ${\mathcal N}(\gamma,X,\rho)$ is finite leading to a finite cover. 
If we consider  the space $\mathcal{Z}$, note that the label set can be partitioned into $|Y|$ sets. Thus, $\mathcal{Z}$ can be partitioned into $|Y|\mathcal{N}(\gamma,X,\rho)$ subsets such that if two instances $z_1=(x_1,y_1)$, $z_2=(x_2,y_2)$ belong to the same subset, then  $y_1=y_1$ and $\rho(x_1,x_2)\leq \gamma$.

\section{Robustness and Generalization for Metric Learning }
\label{robustsec}

\begin{figure}[t]
\begin{center}
\includegraphics[width=0.7\columnwidth]{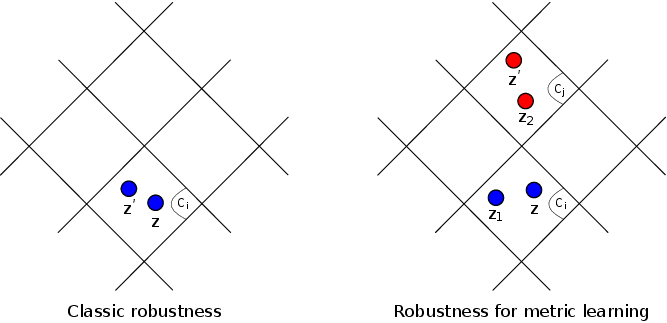}
\caption{Illustration of the robustness property in the classic and metric learning settings. In this example, we use a cover based on the $L_1$ norm. In the classic definition, if any example $z'$ falls in the same region $C_i$ as a training example $z$, then the deviation between their loss must be bounded. In the metric learning definition proposed in this work, for any pair $(z,z')$ and a training pair $(z_1,z_2)$, if $z,z_1$ belong to some region $C_i$ and $z',z_2$ to some region $C_j$, then the deviation between the loss of these two pairs must be bounded.}
\label{fig:robustness}
\end{center}
\end{figure}

We present here our adaptation of robustness to metric learning. 
 The idea is to use the partition of $\mathcal{Z}$ at the pair level: if a new test pair of examples is close to a training pair, then the loss value for each pair must be close. Two pairs are close when each instance of the first pair fall into the same subset of the partition of $\mathcal{Z}$ as the corresponding instance of the other pair, as shown in Figure~\ref{fig:robustness}.
A metric learning algorithm with this property is said robust. This notion is formalized as follows.

\begin{definition}\label{def:robu}
An algorithm $\mathcal{A}$ is $(K,\epsilon(\cdot))$ robust for
$K\in\mathbb{N}$ and $\epsilon(\cdot): (\mathcal{Z}\times\mathcal{Z})^n \rightarrow
\mathbb{R}$ if $\mathcal{Z}$ can be partitioned into $K$ disjoints sets, denoted
by $\{C_i\}_{i=1}^K$, such that  for all
sample $\mathbf{s}\in\mathcal{Z}^n$ and the pair set $p(\mathbf{s})$ associated to this sample, the following holds: \\
$\forall (s_1,s_2) \in p(\mathbf{s}), \forall z_1,z_2 \in
\mathcal{Z}, \forall i,j=1,\ldots,K:$  if 
$s_1,z_1\in C_i$  and  $s_2,z_2\in C_j$ then
\begin{equation}\label{eq:robustness}
|l(\mathcal{A}_{p_{\mathbf{s}}},s_1,s_2)-l(\mathcal{A}_{p_{\mathbf{s}}},z_1,z_2)|\leq \epsilon(p_{\mathbf{s}}).
\end{equation}
\end{definition}

$K$ and $\epsilon(\cdot)$ quantify the robustness of the algorithm and depend on the training sample. The property of robustness is required for every training pair of the sample; we will later see that this property can be relaxed. 

Note that this definition of robustness can be easily extended to triplet based metric learning algorithms. Instead of considering all the pairs $p_{\mathbf{s}}$ from an IID sample $\mathbf{s}$, we take the admissible triplet set $trip_{\mathbf{s}}$ of $\mathbf{s}$ such that $(s_1,s_2,s_3)\in trip_{\mathbf{s}}$ means $s_1$ and $s_2$ share the same label while $s_1$ and $s_3$ have different ones, with the interpretation that $s_1$ must be more similar to $s_2$ than to $s_3$. The robustness property can then be expressed by: 
$\forall (s_1,s_2,s_3) \in trip_{\mathbf{s}}, \forall z_1,z_2,z_3 \in
\mathcal{Z}, \forall i,j,k=1,\ldots,K:$  if 
$s_1,z_1\in C_i$,  $s_2,z_2\in C_j$ and $s_3,z_3\in C_k$ then 
\begin{equation}\label{eq:robu_trip} |l(\mathcal{A}_{trip_{\mathbf{s}}},s_1,s_2,s_3)-l(\mathcal{A}_{trip_{\mathbf{s}}},z_1,z_2,z_3)|\leq \epsilon(trip_{\mathbf{s}}).
\end{equation}

\subsection{Generalization of robust algorithms}
We now give a PAC generalization bound for metric learning algorithms fulfilling the property of robustness (Definition~\ref{def:robu}).
We first  begin by presenting a concentration inequality that will help us to derive the bound.
\begin{proposition}[\cite{Vaart2000}]\label{prop:BHC}
Let $(|N_1|,\ldots,|N_K|)$ an IID multinomial random variable with
parameters $n$ and $(\mu(C_1),\ldots,\mu(C_K))$. 
By the Bretagnolle-Huber-Carol inequality we have:
$
Pr\left\{\sum_{i=1}^K \left|\frac{|N_i|}{n}-\mu(C_i)\right| \geq
  \lambda \right\}\leq 2^K \exp\left(\frac{-n\lambda^2}{2}\right)
$, 
hence with probability at least $1-\delta$,
\begin{equation}
\sum_{i=1}^K \left|\frac{N_i}{n}-\mu(C_i)\right|\leq \sqrt{\frac{2K\ln
  2 + 2 \ln(1/\delta)}{n}}.
\end{equation}
\end{proposition}

We now give our first result on the generalization of metric learning algorithms.
\begin{theorem}\label{th:robu}
If a learning algorithm $\mathcal{A}$ is $(K,\epsilon(\cdot))$-robust
and the training sample is made of the pairs $p_{\mathbf{s}}$ obtained from a sample $\mathbf{s}$ generated by $n$ IID draws from $\mu$, then
for any $\delta>0$, with probability at least $1-\delta$ we have:
$$
|\mathcal{L}(\mathcal{A}_{p_{\mathbf{s}}})-l_{emp}(\mathcal{A}_{p_{\mathbf{s}}})|\leq \epsilon(p_{\mathbf{s}})+2B\sqrt{\frac{2K \ln 2 + 2\ln 1/\delta}{n}}.
$$
\end{theorem}
\begin{proof}
Let $N_i$ be the set of index of points of  $\mathbf{s}$ that fall into the $C_i$. $(|N_1|,\ldots,|N_K|)$ is a IID random variable with parameters $n$ and $(\mu(C_1),\ldots,\mu(C_K))$. 
We have:
\begin{eqnarray*}
\lefteqn{|\mathcal{L}(\mathcal{A}_{p_{\mathbf{s}}})-l_{emp}(\mathcal{A}_{p_{\mathbf{s}}})|}\\
&=&\left|\sum_{i,j=1}^K\mathbb{E}_{z_1,z_2\sim
    \mu}\left(l(\mathcal{A}_{p_{\mathbf{s}}},z_1,z_2)|z_1\in C_i,z_2\in  C_j\right)\mu(C_i)\mu(C_j)-\frac{1}{n^2}\sum_{i,j=1}^nl(\mathcal{A}_{p_{\mathbf{s}}},s_i,s_j)\right|\\
&\stackrel{(a)}{\leq}&\left|\sum_{i,j=1}^K\mathbb{E}_{z_1,z_2\sim
    \mu}\left(l(\mathcal{A}_{p_{\mathbf{s}}},z_1,z_2)|z_1\in C_i,z_2\in
  C_j\right)\mu(C_i)\mu(C_j)-\right.\\
&&\hspace{2cm}\left.\sum_{i,j=1}^K\mathbb{E}_{z_1,z_2\sim
    \mu}\left(l(\mathcal{A}_{p_{\mathbf{s}}},z_1,z_2)|z_1\in C_i,z_2\in
  C_j\right)\mu(C_i)\frac{|N_j|}{n}\right|+\\
&&\left|\sum_{i,j=1}^K\mathbb{E}_{z_1,z_2\sim
    \mu}\left(l(\mathcal{A}_{p_{\mathbf{s}}},z_1,z_2)|z_1\in C_i,z_2\in  C_j\right)\mu(C_i)\frac{|N_j|}{n}-\frac{1}{n^2}\sum_{i,j=1}^nl(\mathcal{A}_{p_{\mathbf{s}}},s_i,s_j)\right|
    \end{eqnarray*}
\begin{eqnarray*}
&\stackrel{(b)}{\leq}&\left|\sum_{i,j=1}^K\mathbb{E}_{z_1,z_2\sim
    \mu}\left(l(\mathcal{A}_{p_{\mathbf{s}}},z_1,z_2)|z_1\in C_i,z_2\in
  C_j\right)\mu(C_i)(\mu(C_j)-\frac{|N_j|}{n})\right|+\\
&&\left|\sum_{i,j=1}^K\mathbb{E}_{z_1,z_2\sim
    \mu}\left(l(\mathcal{A}_{p_{\mathbf{s}}},z_1,z_2)|z_1\in C_i,z_2\in
  C_j\right)\mu(C_i)\frac{|N_j|}{n}-\right.\\
&&\hspace{2cm}\left.\sum_{i,j=1}^K\mathbb{E}_{z_1,z_2\sim
    \mu}\left(l(\mathcal{A}_{p_{\mathbf{s}}},z_1,z_2)|z_1\in C_i,z_2\in
  C_j\right)\frac{|N_i||N_j|}{n}\right|+\\
&&\left|\sum_{i,j=1}^K\mathbb{E}_{z_1,z_2\sim
    \mu}\left(l(\mathcal{A}_{p_{\mathbf{s}}},z_1,z_2)|z_1\in C_i,z_2\in
  C_j\right)\frac{|N_i||N_j|}{n}
-\frac{1}{n^2}\sum_{i,j=1}^nl(\mathcal{A}_{p_{\mathbf{s}}},s_i,s_j)\right|\\
&\stackrel{(c)}{\leq}&B\left(\left|\sum_{j=1}^K\mu(C_j)-\frac{|N_j|}{n}\right|
+\left|\sum_{i=1}^K\mu(C_i)-\frac{|N_i|}{n}\right|\right)+\\
&&\left|\frac{1}{n^2}\sum_{i,j=1}^K\sum_{s_o\in N_i}\sum_{s_l\in
  N_j}\max_{z\in C_i}\max_{z'\in
  C_j}|l(\mathcal{A}_{p_{\mathbf{s}}},z,z')-l(\mathcal{A}_{p_{\mathbf{s}}},s_o,s_l)|\right|\\
&\stackrel{(d)}{\leq}&\epsilon(p_{\mathbf{s}})+2B\sum_{i=1}^K\left|\frac{|N_i|}{n}-\mu(C_i)\right|
\stackrel{(e)}{\leq}\epsilon(p_{\mathbf{s}})+2B\sqrt{\frac{2K \ln 2 + 2\ln 1/\delta}{n}}.
\end{eqnarray*}
Inequalities $(a)$ and $(b)$ are due to the triangle inequality, $(c)$ uses the fact that $l$ is bounded by $B$, that $\sum_{i=1}^K\mu(C_i)=1$  by definition of a multinomial random variable and that $\sum_{j=1}^K \frac{|N_j|}{n}=1$ by definition of the $N_j$. Lastly, $(d)$ is due to the hypothesis of robustness (Equation~\ref{eq:robustness}) and $(e)$ to the application of Proposition~\ref{prop:BHC}.
\end{proof}

The previous bound depends on $K$ which is given by the cover chosen for $\mathcal{Z}$. If for any $K$, the associated $\epsilon(\cdot)$ is a constant (i.e. $\epsilon_K(\mathbf{s})=\epsilon_K$) for any $\mathbf{s}$, we can obtain a bound that holds uniformly for all $K$:
$$|\mathcal{L}(\mathcal{A}_{p_{\mathbf{s}}})-l_{emp}(\mathcal{A}_{p_{\mathbf{s}}})|\leq \inf_{K\geq 1}\left[\epsilon_K+2B\sqrt{\frac{2K \ln 2 + 2\ln1/\delta}{n}}\right].$$

For triplet based metric learning algorithms, by following the definition of robustness given by Equation~\ref{eq:robu_trip} and adapting straightforwardly the losses to triplets such that they output zero for non admissible  ones, Theorem~\ref{th:robu} can be easily extended to obtain the following generalization bound:
\begin{equation}\label{eq:bound_trip}
|\mathcal{L}(\mathcal{A}_{trip_{\mathbf{s}}})-l_{emp}(\mathcal{A}_{trip_{\mathbf{s}}})|\leq \epsilon(trip_{\mathbf{s}})+3B\sqrt{\frac{2K \ln 2 + 2\ln 1/\delta}{n}}.
\end{equation}


\subsection{Pseudo-robustness}

The previous study requires the robustness property to be satisfied for every training pair. In this section, we show that it is possible to relax the robustness such that it must hold only for a subset of the possible pairs, while still providing generalization guarantees. 

\begin{definition}
An algorithm $\mathcal{A}$ is $(K,\epsilon(\cdot),\hat{p}_n(\cdot))$
pseudo-robust for
$K\in\mathbb{N}$, $\epsilon(\cdot): (\mathcal{Z}\times\mathcal{Z})^n \rightarrow
\mathbb{R}$ and $\hat{p}_n(\cdot): (\mathcal{Z}\times\mathcal{Z})^n \rightarrow
\{1,\ldots,n^2\}$, if $\mathcal{Z}$ can be partitioned into $K$ disjoints sets,
denoted 
by $\{C_i\}_{i=1}^K$, such that  for all
$\mathbf{s}\in\mathcal{Z}^n$ IID from $\mu$, there exists a subset of training pairs samples $\hat{p}_{\mathbf{s}} \subseteq p_{\mathbf{s}}$, with $|\hat{p}_{\mathbf{s}}|=\hat{p}_n(p_{\mathbf{s}})$,  
 such that the following holds:\\
$\forall (s_1,s_2)\in \hat{p}_{\mathbf{s}}, \forall z_1,z_2 \in
\mathcal{Z}, \forall i,j=1,\ldots,K$:  if 
$s_1,z_1\in C_i$  and  $s_2,z_2\in C_j$ then
\begin{equation}
|l(\mathcal{A}_{p_{\mathbf{s}}},s_1,s_2)-l(\mathcal{A}_{p_{\mathbf{s}}},z_1,z_2)|\leq \epsilon(p_{\mathbf{s}}).
\end{equation}
\end{definition}
We can easily observe that $(K,\epsilon(\cdot))$-robust is equivalent to $(K,\epsilon(\cdot),n^2)$ pseudo-robust. The following theorem gives the generalization guarantees associated with the pseudo-robustness property.

\begin{theorem}\label{th:pseudo-robustesse}
If a learning algorithm $\mathcal{A}$ is
$(K,\epsilon(\cdot),\hat{p}_n(\cdot))$ pseudo-robust, the training pairs $p_{\mathbf{s}}$ come from a sample  generated by $n$ IID draws from $\mu$, then
for any $\delta>0$, with probability at least $1-\delta$ we have:
$$
|\mathcal{L}(\mathcal{A}_{p_{\mathbf{s}}})-l_{emp}(\mathcal{A}_{p_{\mathbf{s}}})|\leq \frac{\hat{p}_n(p_{\mathbf{s}})}{n^2}\epsilon(p_{\mathbf{s}})+B(\frac{n^2-\hat{p}_n(p_{\mathbf{s}})}{n^2}+2\sqrt{\frac{2K \ln 2 + 2\ln 1/\delta}{n}}).
$$
\end{theorem}

\begin{proof} It is similar to that of Theorem~\ref{th:robu} and is given in \ref{proof:th2}.
\end{proof}

This notion of pseudo-robustness is very relevant to metric learning. Indeed, it is often difficult and potentially damaging to optimize the metric with respect to all possibles pairs, and it has been observed in practice that focusing on a subset of carefully-selected pairs (e.g., defined according to nearest-neighbors) gives much better generalization performance \cite{Weinberger2009,Bellet2012a}. Theorem~\ref{th:pseudo-robustesse} confirms that this principle is well-founded: as long as the robustness property is fulfilled for a (large enough) subset of the pairs, the resulting metric has generalization guarantees. Note that this notion of pseudo-robustness can be also easily adapted to triplet based metric learning.

%
%

\section{Necessity of Robustness}
\label{nessec}

We prove here that a notion of weak robustness is actually necessary and sufficient to generalize in a metric learning setup. 
This result is based on  an asymptotic analysis following the work of Xu and Mannor \cite{Xu2012a}. 
We consider pairs of instances coming from an increasing sample of training instances
$\mathbf{s}=(s_1,s_2,\ldots)$ and from a sample of test instances
$\mathbf{t}=(t_1,t_2,\ldots)$ such that both samples are assumed to be
drawn IID from a distribution $\mu$. We use $\mathbf{s}(n)$ and
$\mathbf{t}(n)$ to denote the first $n$ examples of the two samples respectively, while 
$\mathbf{s}^*$ denotes a fixed sequence of examples. 

We use $L(f,p_{\mathbf{t}(n)})=\frac{1}{n^2} \sum_{(s_i,s_j)\in p_{\mathbf{t}(n)}} l(f,s_i,s_j)$ to refer to the average loss given a set of pairs for any learned metric $f$, and $\mathcal{L}(f)=\mathbb{E}_{z,z'\sim
  \mu}l(f,z,z')$ for the expected loss.

We first define a notion of generalizability for metric learning.
\begin{definition}
Given a training pair set $p_{\mathbf{s}^*}$ coming from a sequence of examples $\mathbf{s}^*$, a metric learning method
$\mathcal{A}$ generalizes w.r.t. $p_{\mathbf{s}^*}$ if
$$\lim_n\left|\mathcal{L}(\mathcal{A}_{p_{\mathbf{s}^*(n)}})-L(\mathcal{A}_{p_{\mathbf{s}^*(n)}},p_{\mathbf{s}^*(n)})\right|=0.$$
Furthermore, a learning method
$\mathcal{A}$ generalizes with probability 1 if it generalizes
with respect to the pairs $p_{\mathbf{s}}$ of almost all samples $\mathbf{s}$  IID from $\mu$.
\end{definition}

Note this notion of generalizability implies convergence in mean. We then introduce the notion of weak robustness for metric learning.
\begin{definition}
Given a set of training pairs $p_{\mathbf{s}^*}$ coming from a sequence of examples ${\mathbf{s}^*}$, a metric learning
  method $\mathcal{A}$ is weakly robust with respect to $p_{\mathbf{s}^*}$ if there
  exists a sequence of $\{\mathcal{D}_n\subseteq \mathcal{Z}^n\}$ such
  that $\Pr(\mathbf{t}(n)\in \mathcal{D}_n)\rightarrow 1$ and
$$
\lim_n\left\{\max_{\mathbf{\hat{s}}(n)\in\mathcal{D}_n}\left|L(\mathcal{A}_{p_{\mathbf{s}^*(n)}},p_{\mathbf{\hat{s}}(n)})-L(\mathcal{A}_{p_{\mathbf{s}^*(n)}},p_{\mathbf{s}^*(n)})\right|\right\}=0.
$$
Furthermore, a learning method $\mathcal{A}$ is almost surely weakly robust if it is robust
 with respect to almost all $\mathbf{s}$.
\end{definition}

Recall that the definition of robustness requires the labeled sample space to be
partitioned into disjoints subsets such that if some instances of pairs of
train/test examples belong to the same partition, then they have
similar loss. Weak robustness is a generalization of this notion where
we consider the average loss of testing and training pairs: if for a
large (in the probabilistic sense) subset of data, the
testing loss is close to the training loss, then the algorithm is
weakly robust. From Proposition~\ref{prop:BHC}, we can see that if for any fixed 
$\epsilon>0$ there exists $K$ such that an algorithm $\mathcal{A}$ is $(K,\epsilon)$ robust, then $\mathcal{A}$ is weakly robust. We now give the main result of this section about the necessity of robustness.

\begin{theorem}\label{th:weak}
Given a fixed sequence of training examples $\mathbf{s}^*$, a metric learning
method $\mathcal{A}$ generalizes with respect to $p_{\mathbf{s}^*}$ if and only if
it is weakly robust with respect to $p_{\mathbf{s}^*}$.
\end{theorem}
\begin{proof2}
Following \cite{Xu2012a}, the sufficiency is obtained by the fact that the testing pairs are obtained from a sample $\mathbf{t}(n)$ constituted of $n$ IID instances. We give the proof in \ref{proof:suff}.

For the necessity, we need the following lemma which is a direct adaptation of a result introduced in \cite{Xu2012a} (Lemma 2). We provide the proof in \ref{proof:lem1} for the sake of completeness. 
\end{proof2}
\begin{lemma}\label{lem:div}
Given $\mathbf{s}^*$, if a learning method is not weakly robust
w.r.t. $p_{\mathbf{s}^*}$, there exists $\epsilon^*,\delta^*>0$ such that
the following holds for infinitely many $n$:
\begin{equation}\label{eq:div}
\Pr(|L(\mathcal{A}_{p_{\mathbf{s}^*(n)}},p_{\mathbf{t}(n)})-L(\mathcal{A}_{p_\mathbf{s}^*(n)},p_{\mathbf{s}^*(n)})|\geq\epsilon^*)\geq \delta^*.
\end{equation}
\end{lemma}
\begin{proof3}

Now, recall that $l$ is positive and uniformly bounded by $B$, thus by the
McDiarmid inequality (recalled in \ref{mcdiarmid}) we have that for any $\epsilon,\delta>0$ there
exists an index $n^*$ such that for any $n>n^*$, with probability at least
$1-\delta$, we have
$|\frac{1}{n^2}\sum_{(t_i,t_j)\in p_{\mathbf{t}(n)}}l(\mathcal{A}_{p_{\mathbf{s}^*(n)}},t_i,t_j)-\mathcal{L}(\mathcal{A}_{p_{\mathbf{s}^*(n)}})|\leq
\epsilon$. This implies the convergence  $\mathcal{L}(\mathcal{A}_{p_{\mathbf{s}^*(n)}},p_{\mathbf{t}(n)})-\mathcal{L}(\mathcal{A}_{p_{\mathbf{s}^*(n)}})\stackrel{Pr}{\rightarrow}0$, and thus from a given index:
\begin{equation}\label{eq:lim}
|\mathcal{L}(\mathcal{A}_{p_{\mathbf{s}^*(n)}},p_{\mathbf{t}(n)})-\mathcal{L}(\mathcal{A}_{p_{\mathbf{s}^*(n)}})|\leq \frac{\epsilon^*}{2}.
\end{equation}

Now, by contradiction, suppose algorithm $\mathcal{A}$ is not weakly robust, Lemma~\ref{lem:div}
implies Equation \ref{eq:div} holds for infinitely many $n$. This combined with
Equation~\ref{eq:lim} implies that for infinitely many $n$:
$$
|\mathcal{L}(\mathcal{A}_{p_{\mathbf{s}^*(n)}},p_{\mathbf{t}(n)})-L(\mathcal{A}_{p_{\mathbf{s}^*(n)}},p_{\mathbf{s}^*(n)})|\geq \frac{\epsilon^*}{2}
$$
which means $\mathcal{A}$ does not generalize, thus the necessity of
weak robustness is established.
\end{proof3}

The following corollary follows immediately from Theorem~\ref{th:weak}.
\begin{corollary}
A metric learning method $\mathcal{A}$ generalizes with probability 1 if and
only if it is almost surely weakly robust.
\end{corollary}

\section{Examples of Robust Metric Learning Algorithms}
\label{exsec}

We first restrict our attention to Mahalanobis distance learning algorithms of the following form:
\begin{eqnarray}
\displaystyle\min_{\mathbf{M} \succeq 0} & c\|\mathbf{M}\| + \frac{1}{n^2}\displaystyle\sum_{(s_i,s_j)\in p_s}g(y_{ij}[1-f(\mathbf{M},x_i,x_j)]),
\label{generalform}
\end{eqnarray}
where $s_i=(x_i,y_i)$, $s_j=(x_j,y_j)$, $y_{ij} = 1$ if $y_i=y_j$ and $-1$ otherwise, $f(\mathbf{M},x_i,x_j) = (x_i-x_j)^T\mathbf{M}(x_i-x_j)$ is the Mahalanobis distance parameterized by the $d\times d$ PSD matrix $\mathbf{M}$, $\|\cdot\|$ some matrix norm and $c$ a regularization parameter. The loss function $l(f,s_i,s_j) = g(y_{ij}[1-f(\mathbf{M},x_i,x_j)])$ outputs a small value when its input is large positive and a large value when it is large negative. We assume $g$ to be nonnegative and Lipschitz continuous with Lipschitz constant $U$. Lastly, $g_0 = \sup_{s_i,s_j}g(y_{ij}[1-f(\mathbf{0},x_i,x_j)])$ is the largest loss when $\mathbf{M}$ is $\mathbf{0}$.
The general form \eqref{generalform} encompasses many existing metric learning formulations. For instance, in the case of the hinge loss and Frobenius norm regularization, we recover \cite{Jin2009}, while the family of formulations studied in \cite{Kunapuli2012} corresponds to a trace norm regularizer.

To prove the robustness of \eqref{generalform}, we will use the following theorem, which is based on the geometric intuition behind robustness. It essentially says that if a metric learning algorithm achieves approximately the same testing loss for testing pairs that are close to each other, then it is robust.\footnote{We provide a similar theorem for the case of triplets in \ref{tripletthm}.}
\begin{theorem}
Fix $\gamma>0$ and a metric $\rho$ of $\mathcal{Z}$. Suppose $\mathcal{A}$ satisfies:\\
$\forall z_1,z_2,z'_1,z'_2 : z_1,z_2\in \mathbf{s}, \rho(z_1,z'_1)\leq \gamma, \rho(z_2,z'_2)\leq \gamma$,
$$|l(\mathcal{A}_{p_\mathbf{s}},z_1,z_2)-l(\mathcal{A}_{p_\mathbf{s}},z'_1,z'_2)|\leq \epsilon(p_\mathbf{s})$$
and $\mathcal{N}(\gamma/2,\mathcal{Z},\rho) < \infty$. Then $\mathcal{A}$ is $(\mathcal{N}(\gamma/2,\mathcal{Z},\rho),\epsilon(p_\mathbf{s}))$-robust.
\label{testtheorem}
\end{theorem}
\begin{proof}
By definition of covering number, we can partition ${X}$ in $\mathcal{N}(\gamma/2,{X},\rho)$ subsets such that each subset has a diameter less or equal to $\gamma$. Furthermore, since ${Y}$ is a finite set, we can partition $\mathcal{Z}$ into $|Y|\mathcal{N}(\gamma/2,{X},\rho)$ subsets $\{C_i\}$ such that $z_1,z'_1\in C_i \Rightarrow \rho(z_1,z'_1)\leq \gamma$.
Therefore, $\forall z_1,z_2,z'_1,z'_2 : z_1,z_2\in \mathbf{s}, \rho(z_1,z'_1)\leq \gamma, \rho(z_2,z'_2)\leq \gamma$, 
$$|l(\mathcal{A}_{p_\mathbf{s}},z_1,z_2)-l(\mathcal{A}_{p_\mathbf{s}},z'_1,z'_2)|\leq \epsilon(p_\mathbf{s}),$$
this implies
$z_1,z_2\in \mathbf{s}, z_1,z'_1\in C_i,z_2,z'_2\in C_j \Rightarrow |l(\mathcal{A}_{p_\mathbf{s}},z_1,z_2)-l(\mathcal{A}_{p_\mathbf{s}},z'_1,z'_2)|\leq \epsilon(p_\mathbf{s}),$
which establishes the theorem.
\end{proof}

This theorem provides a roadmap for deriving generalization guarantees based on the robustness framework. Indeed, given a partition of the input space, one must bound the deviation between the loss for any pair of examples with corresponding elements belonging to the same partitions. This bound is generally a constant that depends on the problem to solve and the thinness of the partition defined by $\gamma$. This bound tends to zero as $\gamma\rightarrow 0$, which ensures the consistency of the approach. While this framework is rather general, the price to pay is the relative looseness of the bounds, as discussed in Section~\ref{disc}.

Recall that we assume that $\forall x\in X$, $\|x\|\leq R$ for some convenient norm $\|\cdot\|$. Following Theorem~\ref{testtheorem}, we now prove the robustness of \eqref{generalform} when $\|\mathbf{M}\|$ is the Frobenius norm.
\begin{example}[Frobenius norm]
Algorithm \eqref{generalform} with the Frobenius norm  $\|\mathbf{M}\| = \|\mathbf{M}\|_{\mathcal{F}} = \sqrt{\sum_{i=1}^d\sum_{j=1}^d m_{ij}^2}$ is $(|Y|\mathcal{N}(\gamma/2,{X},\|\cdot\|_2),\frac{8UR\gamma g_0}{c})$-robust.
\label{ex1}
\end{example}
\begin{proof}
Let $\mathbf{M^*}$ be the solution given training data $p_s$. Thus, due to optimality of $\mathbf{M^*}$, we have
\begin{eqnarray*}
\lefteqn{c\|\mathbf{M^*}\|_{\mathcal{F}} + \frac{1}{n^2}\displaystyle\sum_{(s_i,s_j)\in p_s}g(y_{ij}[1-f(\mathbf{M},x_i,x_j)]) \leq} \\
&\ &c\|\mathbf{0}\|_{\mathcal{F}} + \frac{1}{n^2}\displaystyle\sum_{(s_i,s_j)\in p_s}g(y_{ij}[1-f(\mathbf{0},x_i,x_j)]) = g_0
\end{eqnarray*}
and thus $\|\mathbf{M^*}\|_{\mathcal{F}} \leq g_0/c$.
We can partition $\mathcal{Z}$ as $|Y|\mathcal{N}(\gamma/2,{X},\|\cdot\|_2)$ sets, such that if $z$ and $z'$ belong to the same set, then $y=y'$ and $\|x-x'\|_2 \leq \gamma$. Now, for $z_1,z_2,z'_1,z'_2\in\mathcal{Z}$, if $y_1=y'_1$, $\|x_1-x'_1\|_2 \leq \gamma$, $y_2=y'_2$ and $\|x_2-x'_2\|_2 \leq \gamma$, then:
\begin{eqnarray*}
\lefteqn{|g(y_{12}[1-f(\mathbf{M^*},x_1,x_2)]) - g(y'_{12}[1-f(\mathbf{M^*},x'_1,x'_2)])|}\\
& \leq & U|(x_1-x_2)^T\mathbf{M^*}(x_1-x_2)-(x'_1-x'_2)^T\mathbf{M^*}(x'_1-x'_2)|\\
& = & U|(x_1-x_2)^T\mathbf{M^*}(x_1-x_2)-(x_1-x_2)^T\mathbf{M^*}(x'_1-x'_2)\\
& & + ~(x_1-x_2)^T\mathbf{M^*}(x'_1-x'_2)|-(x'_1-x'_2)^T\mathbf{M^*}(x'_1-x'_2)|\\
& = & U|(x_1-x_2)^T\mathbf{M^*}(x_1-x_2-(x'_1+x'_2)) +\\ 
& &(x_1-x_2-(x'_1+x'_2))^T\mathbf{M^*}(x'_1+x'_2)|\\
& \leq & U (|(x_1-x_2)^T\mathbf{M^*}(x_1-x'_1)| + |(x_1-x_2)^T\mathbf{M^*}(x'_2-x_2)|\\
& & + ~|(x_1-x'_1)^T\mathbf{M^*}(x'_1+x'_2)| + |(x'_2-x_2)^T\mathbf{M^*}(x'_1+x'_2)|)\\
& \leq & U(\|x_1-x_2\|_2\|\mathbf{M^*}\|_{\mathcal{F}}\|x_1-x'_1\|_2 + \|x_1-x_2\|_2\|\mathbf{M^*}\|_{\mathcal{F}}\|x'_2-x_2\|_2\\
& & + ~\|x_1-x'_1\|_2\|\mathbf{M^*}\|_{\mathcal{F}}\|x'_1-x'_2\|_2 + \|x'_2-x_2\|_2\|\mathbf{M^*}\|_{\mathcal{F}}\|x'_1-x'_2\|_2)\\
& \leq&  \frac{8UR\gamma g_0}{c}
\end{eqnarray*}
Hence, the example holds by Theorem~\ref{testtheorem}.
\end{proof}

Note that for the special case of Example~\ref{ex1}, a generalization bound (with same order of convergence rate) based on uniform stability was derived in \cite{Jin2009}.
However, it is known that sparse algorithms are not stable \cite{Xu2012}, and thus stability-based analysis fails to assess the generalization ability of recent sparse metric learning approaches \cite{Rosales2006,Qi2009,Ying2009,McFee2010,Kunapuli2012}.
The key advantage of robustness over stability is that we can obtain bounds similar to the Frobenius case for arbitrary $p$-norms (or even any regularizer which is bounded below by some $p$-norm) using equivalence of norms arguments. To illustrate this, we show the robustness when $\|\mathbf{M}\|$ is the $\ell_1$ norm (used in \cite{Rosales2006,Qi2009}) which promotes sparsity at the entry level, the $\ell_{2,1}$ norm (used e.g. in \cite{Ying2009}) which induces sparsity at the column/row level, and the trace norm (used e.g. in \cite{McFee2010,Kunapuli2012}) which favors low-rank matrices.\footnote{In the last two cases, the linear projection space of the data induced by the learned Mahalanobis distance is of lower dimension than the original space, allowing more efficient computations and reduced memory usage.} The proofs are reminiscent of that of Example~\ref{ex1} and can be found in \ref{proof:ex2} and \ref{proof:ex3}, respectively.

\begin{example}[$\ell_1$ norm]
Algorithm \eqref{generalform} with $\|\mathbf{M}\| = \|\mathbf{M}\|_1$ is $(|Y|\mathcal{N}(\gamma,\mathcal{X},\|\cdot\|_1),\frac{8UR\gamma g_0}{c})$-robust.
\label{ex2}
\end{example}

\begin{example}[$\ell_{2,1}$ norm and trace norm]
Consider Algorithm \eqref{generalform} with\newline $\|\mathbf{M}\|~=~\|\mathbf{M}\|_{2,1}~=~\sum_{i=1}^d \|m^i\|_2$, where $m^i$ is the $i$-th column of $\mathbf{M}$. This algorithm is $(|Y|\mathcal{N}(\gamma,\mathcal{X},\|\cdot\|_2),\frac{8UR\gamma g_0}{c})$-robust. The same holds for the trace norm $\|\mathbf{M}\|_*$, which is the sum of the singular values of $\mathbf{M}$.
\label{ex3}
\end{example}

Some metric learning algorithms have  kernelized versions, for instance~\cite{Schultz2003,Davis2007}. In the following example we show robustness for a kernelized formulation. The proof can be found in \ref{proof:ex4}.
\begin{example}[Kernelization]\label{ex:kernel} Consider the kernelized version of \eqref{generalform}:
\begin{eqnarray}
\displaystyle\min_{\mathbf{M} \succeq 0} & c\|\mathbf{M}\|_{\mathbb{H}} + \frac{1}{n^2}\displaystyle\sum_{(s_i,s_j)\in p_s}g(y_{ij}[1-f(\mathbf{M},\phi(x_i),\phi(x_j))]),
\label{generalform_kernel}
\end{eqnarray}
where $\phi(\cdot)$ is a feature mapping to a kernel space $\mathbb{H}$,
$\|\cdot\|_{\mathbb{H}}$ the norm function of $\mathbb{H}$ and
$k(\cdot,\cdot)$ the kernel function. 
Consider a cover of $X$ by $\|\cdot\|_2$ and let
$f_{\mathbb{H}}(\gamma)\stackrel{\bigtriangleup}{=}\max_{a,b\in X,
  \|a-b\|_2\leq \gamma}(k(a,a)+k(b,b)-2k(a,b))$ and
$B_\gamma=\max_{x\in X}\sqrt{k(x,x)}$. If the kernel function is
continuous, $B_\gamma$ and $f_{\mathbb{H}}$ are finite for any
$\gamma>0$ and thus Algorithm~\ref{generalform_kernel} is
$(|Y|\mathcal{N}(\gamma,X,\|\cdot\|_2),\frac{8 U B_\gamma
  \sqrt{f_{\mathbb{H}}} g_0}{c})$-robust.
  \label{ex4}
\end{example}

Finally, the flexibility of our framework allows us to derive bounds for other forms of metric as well as for formulations based on triplet constraints using the same proof techniques as above. We illustrate this in Example~\ref{ex5} and Example~\ref{ex6}, and for the sake of completeness we provide the proofs in \ref{proof:ex5} and \ref{proof:ex6} respectively.

\begin{example} Consider Algorithm \eqref{generalform} with bilinear similarity $f(\mathbf{M},x_i,x_j) = x_i^T\mathbf{M}x_j$ instead of the Mahalanobis distance, as studied in \cite{Chechik2009,Qamar2010,Bellet2012a}. For the regularizers considered in Examples~\ref{ex1} -- \ref{ex3}, we can improve the robustness to $2UR\gamma g_0/c$ (due to the simpler form of the bilinear similarity).
\label{ex5}
\end{example}

\begin{example}
 Using triplet-based robustness (Equation~\ref{eq:robu_trip}), we can show the robustness of two popular triplet-based metric learning approaches \cite{Schultz2003,Ying2009} for which no generalization guarantees were known (to the best of our knowledge). These algorithms have the following form:
$$\displaystyle\min_{\mathbf{M} \succeq 0} c\|\mathbf{M}\| + \frac{1}{|trip_\mathbf{s}|}\displaystyle\sum_{(s_i,s_j,s_k)\in trip_\mathbf{s}}[1-(x_i-x_k)^T\mathbf{M}(x_i-x_k)+(x_i-x_j)^T\mathbf{M}(x_i-x_j)]_+,$$
where $\|\mathbf{M}\|$ = $\|\mathbf{M}\|_{\mathcal{F}}$ in \cite{Schultz2003} or $\|\mathbf{M}\| = \|\mathbf{M}\|_{1,2}$ in \cite{Ying2009}. These methods are $(\mathcal{N}(\gamma,\mathcal{Z},\|\cdot\|_2),\frac{16UR\gamma g_0}{c})$-robust (the additional factor 2 comes from the use of triplets instead of pairs).
\label{ex6}
\end{example}



\section{Discussion}
\label{disc}

This section discusses the bounds derived from
the proposed framework and put then into perspective with other approaches. 

As seen in the previous section, our approach is rather general and allows one to
derive generalization bounds for many metric learning
methods. The counterpart of this generality is the relative looseness of
the resulting bounds: although the $O(1/\sqrt{n})$ convergence rate is the same as in the alternative frameworks presented below, the covering number constants are difficult to estimate and can be large. Therefore, these bounds are useful to establish the
consistency of a metric learning approach but do not provide sharp estimates of the generalization loss. This is in accordance with the original robustness bounds introduced in \cite{Xu2010,Xu2012a}.

The guarantees proposed in \cite{Bian2011,Bian2012} can be tighter but hold only under strong assumptions on the distribution of examples. Morever, these results only apply to a specific metric learning formulation and it is not clear how they can be adapted to more general forms. Bounds based on uniform stability \cite{Jin2009} are also tighter and can deal with various loss functions, but fail to address sparsity-inducing regularizers. This is known to be a general limitation of stability-based analysis \cite{Xu2012}.


More recently, independently and in parallel to our work, generalization bounds for metric learning based on Rademacher analysis have
been proposed \cite{Cao2012a,Guo2014}. These bounds are
 tighter than the ones obtained with robustness and can tackle some sparsity-inducing regularizers. Their derivation is however more involved as it requires to compute
Rademacher average estimates related to the matrix dual norm. For this reason, their analysis is limited to matrix norm regularization, while our framework can essentially accommodate any regularizer that is bounded below by some matrix $p$-norm (following the same proof technique as in Section~\ref{exsec}). Furthermore, robustness is flexible enough to tackle other settings (such as triplet-based constraints), as illustrated in Section~\ref{exsec}.

We conclude this discussion by noting that the proposed framework can be used to obtain generalization bounds
for linear classifiers that use the learned metrics, following the work of
\cite{Bellet2012a,Guo2014}. 

\section{Conclusion}
\label{conclu}

We proposed a new theoretical framework for evaluating the
generalization ability of metric learning based on the notion of
algorithm robustness originally introduced in \cite{Xu2012a}. 
We showed that a weak notion of robustness characterizes the
generalizability of metric learning algorithms, justifying that
robustness is fundamental for such algorithms. 
The proposed framework has an intuitive geometric meaning and allows us to derive generalization bounds for a large
class of algorithms with different regularizations (such as sparsity
inducing norms), showing that it has a wider applicability than existing frameworks. 
 Moreover, few algorithm-specific arguments are needed. The price to pay is the relative looseness of the resulting bounds.
 
A perspective of this work is to take advantage of the generality and flexibility of the robustness framework to tackle more complex metric learning settings, for instance other regularizers regularizers (such as the LogDet divergence used in \cite{Davis2007,Jain2008}), methods that learn multiple metrics (e.g., \cite{Wang2012b,Shi2014}), and metric learning for domain adaptation \cite{Kulis2011,Geng2011}. It is also promising to investigate whether robustness could be used to derive guarantees for online algorithms such as \cite{Shalev-Shwartz2004,Jain2008,Chechik2009}.

Another exciting direction for future work is to investigate new metric learning algorithms based on the robustness property. For instance, given a partition of the labeled input space and for any two regions, such an algorithm could minimize the maximum loss over pairs of examples belonging to each region. This is reminiscent of concepts from robust optimization \cite{Ben-Tal2009} and could be useful to deal with noisy settings.

\appendix

\section{Appendix}

\section{Proof of  Theorem~\ref{th:pseudo-robustesse} (pseudo-robustness)}
\label{proof:th2}

\begin{proof}
From the proof of Theorem~\ref{th:robu}, we can easily deduce that:
\begin{small}
\begin{eqnarray*}
\lefteqn{|\mathcal{L}(\mathcal{A}_{p_{\mathbf{s}}})-l_{emp}(\mathcal{A}_{p_{\mathbf{s}}})|\leq 2B\sum_{i=1}^K|\frac{|N_i|}{n}-\mu(C_i)|+}\\
&\left|\sum_{i,j=1}^K\mathbb{E}_{z_1,z_2\sim
    \mu}\left(l(\mathcal{A}_{p_{\mathbf{s}}},z_1,z_2\right)|z_1\in C_i,z_2\in
  C_j)\frac{|N_i||N_j|}{n}
-\frac{1}{n^2}\sum_{i,j=1}^nl(\mathcal{A}_{p_{\mathbf{s}}},s_i,s_j)\right|.&
\end{eqnarray*}
\end{small}
Then, we have
\begin{eqnarray*}
&\leq&2B\sum_{i=1}^K|\frac{|N_i|}{n}-\mu(C_i)| +\\
&&\left|\frac{1}{n^2}\sum_{i,j=1}^K\sum_{(s_o,s_l)\in\hat{p}(\mathbf{s})}
\sum_{s_o\in N_i}\sum_{s_l\in  N_j}\max_{z\in C_i}\max_{z'\in C_j}|l(\mathcal{A}_{p_{\mathbf{s}}},z,z')-l(\mathcal{A}_{p_{\mathbf{s}}},s_o,s_l)|\right|+\\
&&\left|\frac{1}{n^2}\sum_{i,j=1}^K\sum_{(s_o,s_l)\not\in\hat{p}(\mathbf{s})}\sum_{s_o\in N_i}\sum_{s_l\in  N_j}\max_{z\in C_i}\max_{z'\in C_j}|l(\mathcal{A}_{p_{\mathbf{s}}},z,z')- l(\mathcal{A}_{p_{\mathbf{s}}},s_o,s_l)|\right|\\
&\leq&\frac{\hat{p}_n(p_{\mathbf{s}})}{n^2}\epsilon(p_{\mathbf{s}})+B\left(\frac{n^2-\hat{p}_n(p_{\mathbf{s}})}{n^2}+2\sqrt{\frac{2K \ln 2 + 2\ln 1/\delta}{n}}\right).
\end{eqnarray*}
The second inequality is obtained by the triangle inequality, the last one is obtained by the application of Proposition~\ref{prop:BHC}, the hypothesis of pseudo-robustness and the fact that $l$ is positive and bounded by $B$, thus we have $|l(\mathcal{A}_{p_{\mathbf{s}}},z,z')-l(\mathcal{A}_{p_{\mathbf{s}}},s_o,s_l)|\leq B$.
\end{proof}

\section{Proof of sufficiency of Theorem~\ref{th:weak}}
\label{proof:suff}

\begin{proof} The proof of sufficiency closely follows the first part of the proof of Theorem 8 in \cite{Xu2012a}. 
When  $\mathcal{A}$ is weakly robust, there exits a sequence $\{{D}_n\}$
such that for any $\delta,\epsilon>0$ there exists
$N(\delta,\epsilon)$ such that for all $n>N(\delta,\epsilon)$,
$\Pr(\mathbf{t}(n)\in D_n)>1-\delta$ and
\begin{equation}\label{eq:d_n}
\max_{\mathbf{\hat{s}}(n)\in {D}_n}\left|L(\mathcal{A}_{p_{\mathbf{s}^*(n)}},p_{\mathbf{\hat{s}}(n)})-L(\mathcal{A}_{p_{\mathbf{s}^*(n)}},p_{\mathbf{{s}}^*(n)})\right|<\epsilon.
\end{equation} 
Therefore for any  $n>N(\delta,\epsilon)$,
\begin{eqnarray*}
\lefteqn{|\mathcal{L}(\mathcal{A}_{p_{\mathbf{s}^*(n)}})-L(\mathcal{A}_{p_{\mathbf{s}^*(n)}},p_{\mathbf{s}^*(n)})|}\\
&=&|\mathbb{E}_{\mathbf{t}(n)}(L(\mathcal{A}_{p_{\mathbf{s}^*(n)}},p_{\mathbf{t}(n)}))-L(\mathcal{A}_{p_{\mathbf{s}^*(n)}},p_{\mathbf{s}^*(n)})|\\\
&=&|\Pr(\mathbf{t}(n)\not\in D_n)\mathbb{E}(L(\mathcal{A}_{p_{\mathbf{s}^*(n)}},p_{\mathbf{t}(n)})|\mathbf{t}(n)\not\in
D_n)\\
&&+\Pr(\mathbf{t}(n)\in D_n)\mathbb{E}(L(\mathcal{A}_{p_{\mathbf{s}^*(n)}},p_{\mathbf{t}(n)})|\mathbf{t}(n)\in
D_n)-
L(\mathcal{A}_{p_{\mathbf{s}^*(n)}},p_{\mathbf{s}^*(n)})|\\
&\leq&\Pr(\mathbf{t}(n)\not\in D_n)|\mathbb{E}(L(\mathcal{A}_{p_{\mathbf{s}^*(n)}},p_{\mathbf{t}(n)})|\mathbf{t}(n)\not\in
D_n)-L(\mathcal{A}_{p_{\mathbf{s}^*(n)}},p_{\mathbf{s}^*(n)})|+\\
&&\Pr(\mathbf{t}(n)\in D_n)|\mathbb{E}(L(\mathcal{A}_{p_{\mathbf{s}^*(n)}},p_{\mathbf{t}(n)})|\mathbf{t}(n)\in
D_n)-L(\mathcal{A}_{p_{\mathbf{s}^*(n)}},p_{\mathbf{s}^*(n)})|\\
&\leq&\delta
B+\max_{\mathbf{\hat{s}}(n)\in\mathcal{D}_n}|L(\mathcal{A}_{p_{\mathbf{{s}^*}(n)}},p_{\mathbf{\hat{s}}(n)})-L(\mathcal{A}_{p_{\mathbf{s}^*(n)}},p_{\mathbf{s}^*(n)})|\\ &\leq&
\delta B+\epsilon.
\end{eqnarray*}
The first inequality holds because the testing samples $\mathbf{t}(n)$
consist of $n$ instances IID from $\mu$. The second equality
is obtained by conditional expectation. The next inequality uses the positiveness and the upper bound $B$ of the loss function. Finally, we apply 
Equation~\ref{eq:d_n}. 
We thus conclude that $\mathcal{A}$ generalizes for $p_{\mathbf{s}^*}$
because $\epsilon$ and $\delta$ can be chosen arbitrary.
\end{proof}

\section{Proof of Lemma~\ref{lem:div}}
\label{proof:lem1}

\begin{proof}
This proof follows the same principle as the proof of Lemma 2 from \cite{Xu2012a}.  
By contradiction, assume $\epsilon^*$ and $\delta^*$ do not exist. Let
$\epsilon_v=\delta_v=1/v$ for $v=1,2, ...$, then there exists a non
decreasing sequence $\{N(v)\}_{v=1}^\infty$ such that for all $v$, if
$n\geq N(v)$ then
$\Pr(|L(\mathcal{A}_{p_{\mathbf{s}^*(n)}},p_{\mathbf{t}(n)})-L(\mathcal{A}_{p_{\mathbf{s}^*(n)}},p_{\mathbf{s}^*(n)})|\geq
\epsilon_v)<\delta_v$. 
For each $n$ we define 
$$
D_n^v\triangleq\{\mathbf{\hat{s}}(n)|L(\mathcal{A}_{p_{\mathbf{s}^*(n)}},p_{\mathbf{\hat{s}}(n)})-L(\mathcal{A}_{p_{\mathbf{s}^*(n)}},p_{\mathbf{s}^*(n)})|<\epsilon_v\}.
$$
For each $n\geq N(v)$ we have 
$$\Pr(\mathbf{t}(n)\in
D_n^v)=1-\Pr(|L(\mathcal{A}_{p_{\mathbf{s}^*(n)}},p_{\mathbf{t}(n)})-L(\mathcal{A}_{p_{\mathbf{s}^*(n)}},p_{\mathbf{s}^*(n)})|\geq
\epsilon_v)>1-\delta_v.$$
For $n\geq N(1)$, define $D_n\triangleq D_n^{v(n)}$, where $v(n)=\max(v|N(v)\leq
n; v\leq n)$. Thus for all, $n\geq N(1)$ we have $\Pr(\mathbf{t}(n)\in
D_n)>1-\delta_{v(n)}$ and 
$$\sup_{\mathbf{\hat{s}}(n)\in D_n}|L(\mathcal{A}_{p_{\mathbf{s}^*(n)}},p_{\mathbf{\hat{s}}(n)})-L(\mathcal{A}_{p_{\mathbf{s}^*(n)}},p_{\mathbf{s}^*(n)})|<\epsilon_{v(n)}.$$ 
Note that $v(n)$ tends to infinity, it follows that
$\delta_{v(n)}\rightarrow 0$ and $\epsilon_{v(n)}\rightarrow 0$. 
Therefore, $\Pr(\mathbf{t}(n)\in D_n)\rightarrow 1$ and 
$$
\lim_{n\rightarrow \infty}\{\sup_{\mathbf{\hat{s}}(n)\in D_n}|L(\mathcal{A}_{p_{\mathbf{s}^*(n)}},p_{\mathbf{\hat{s}}(n)})-L(\mathcal{A}_{p_{\mathbf{s}^*(n)}},p_{\mathbf{s}^*(n)})|\}=0.
$$
That is $\mathcal{A}$ is weakly robust. w.r.t. $p_{\mathbf{s}}$ which is a
desired contradiction.
\end{proof}

\section{Mc Diarmid inequality}
\label{mcdiarmid}
 
Let $X_1, \ldots, X_n$ be $n$ independent random variables taking values in
 ${X}$ and let $Z=f(X_1, \ldots, X_n)$. If for each $1\leq i \leq n$, there exists a constant $c_i$ such that
\begin{eqnarray*}
\lefteqn{\sup_{x_1, \ldots, x_n, x'_i\in \mathcal{X}}|f(x_1, \ldots, x_i, \ldots, x_n) - f(x_1, \ldots,
 x'_i, \ldots, x_n)|\leq c_i, \forall 1\leq i\leq n,}\\
&&\textrm{then for any } \epsilon>0,\quad\quad\quad \mathrm{Pr}[|Z-{\mathbb E}[Z]|\geq \epsilon]\leq
2\exp\left(\frac{-2\epsilon^2}{\sum_{i=1}^n c_i^2}\right).
\end{eqnarray*}

\section{Robustness Theorem for Triplet-based Approaches}
\label{tripletthm}
We give here an adaptation of Theorem~\ref{testtheorem} for
triplet-based approaches. The proof follows the same
principle as the one of Theorem~\ref{testtheorem}. 
\begin{theorem}
Fix $\gamma>0$ and a metric $\rho$ of $\mathcal{Z}$. Suppose $\mathcal{A}$ satisfies:\\
$\forall z_1,z_2,z_3,,z'_1,z'_2,z'_3 : z_1,z_2,z_3\in \mathbf{s},
\rho(z_1,z'_1)\leq \gamma, \rho(z_2,z'_2)\leq
\gamma,\rho(z_3,z'_3)\leq \gamma$, 
$$|l(\mathcal{A}_{trip_\mathbf{s}},z_1,z_2,z_3)-l(\mathcal{A}_{tripp_\mathbf{s}},z'_1,z'_2,z'_3)|\leq \epsilon(trip_\mathbf{s})$$
and $\mathcal{N}(\gamma/2,\mathcal{Z},\rho) < \infty$. Then $\mathcal{A}$ is $(\mathcal{N}(\gamma/2,\mathcal{Z},\rho),\epsilon(trip_\mathbf{s}))$-robust.
\label{testtheoremtrip}
\end{theorem}

\section{Proof of Example~\ref{ex2} ($\ell_1$ norm)}
\label{proof:ex2}

\begin{proof}
Let $\mathbf{M^*}$ be the solution given training data $p_{\mathbf{s}}$. Due to optimality of $\mathbf{M^*}$, we have $\|\mathbf{M^*}\|_1 \leq g_0/c$.
We can partition $\mathcal{Z}$ as $|Y|\mathcal{N}(\gamma/2,{X},\|\cdot\|_1)$ sets, such that if $z$ and $z'$ belong to the same set, then $y=y'$ and $\|x-x'\|_1 \leq \gamma$. Now, for $z_1,z_2,z'_1,z'_2\in\mathcal{Z}$, if $y_1=y'_1$, $\|x_1-x'_1\|_1 \leq \gamma$, $y_2=y'_2$ and $\|x_2-x'_2\|_1 \leq \gamma$, then:
\begin{eqnarray*}
\lefteqn{|g(y_{12}[1-f(\mathbf{M^*},x_1,x_2)]) - g(y'_{12}[1-f(\mathbf{M^*},x'_1,x'_2)])|}\\
& \leq & U (|(x_1-x_2)^T\mathbf{M^*}(x_1-x'_1)| + |(x_1-x_2)^T\mathbf{M^*}(x'_2-x_2)|\\
& & + ~|(x_1-x'_1)^T\mathbf{M^*}(x'_1+x'_2)| + |(x'_2-x_2)^T\mathbf{M^*}(x'_1+x'_2)|)\\
& \leq & U(\|x_1-x_2\|_\infty\|\mathbf{M^*}\|_1\|x_1-x'_1\|_1 + \|x_1-x_2\|_\infty\|\mathbf{M^*}\|_1\|x'_2-x_2\|_1\\
& & + ~\|x_1-x'_1\|_1\|\mathbf{M^*}\|_1\|x'_1-x'_2\|_\infty + \|x'_2-x_2\|_1\|\mathbf{M^*}\|_1\|x'_1-x'_2\|_\infty)\\
& \leq & \frac{8UR\gamma g_0}{c}.
\end{eqnarray*}
\end{proof}

\section{Proof of Example~\ref{ex3} ($\ell_{2,1}$ norm and trace norm)}
\label{proof:ex3}

\begin{proof}
Let $\mathbf{M^*}$ be the solution given training data $p_{\mathbf{s}}$. Due to optimality of $\mathbf{M^*}$, we have $\|\mathbf{M^*}\|_{2,1} \leq g_0/c$. We can partition $\mathcal{Z}$ in the same way as in the proof of Example~\ref{ex1} and use the inequality $\|\mathbf{M^*}\|_{\mathcal{F}} \leq \|\mathbf{M^*}\|_{2,1}$ (from Theorem~3 of \cite{Feng2003,Klaus1995}) to derive the same bound:
\begin{eqnarray*}
\lefteqn{|g(y_{12}[1-f(\mathbf{M^*},x_1,x_2)]) - g(y'_{12}[1-f(\mathbf{M^*},x'_1,x'_2)])|}\\
& \leq & U(\|x_1-x_2\|_2\|\mathbf{M^*}\|_{\mathcal{F}}\|x_1-x'_1\|_2 + \|x_1-x_2\|_2\|\mathbf{M^*}\|_{\mathcal{F}}\|x'_2-x_2\|_2\\
& & + ~\|x_1-x'_1\|_2\|\mathbf{M^*}\|_{\mathcal{F}}\|x'_1-x'_2\|_2 + \|x'_2-x_2\|_2\|\mathbf{M^*}\|_{\mathcal{F}}\|x'_1-x'_2\|_2)\\
& \leq & U(\|x_1-x_2\|_2\|\mathbf{M^*}\|_{2,1}\|x_1-x'_1\|_2 + \|x_1-x_2\|_2\|\mathbf{M^*}\|_{2,1}\|x'_2-x_2\|_2\\
& & + ~\|x_1-x'_1\|_2\|\mathbf{M^*}\|_{2,1}\|x'_1-x'_2\|_2 + \|x'_2-x_2\|_2\|\mathbf{M^*}\|_{2,1}\|x'_1-x'_2\|_2)\\
& \leq & \frac{8UR\gamma g_0}{c}.
\end{eqnarray*}
For the trace norm,  we use the classic result $\|\mathbf{M^*}\|_{\mathcal{F}}\leq \|\mathbf{M}\|_*$, which allows us to prove the same result by replacing $\|\cdot\|_{2,1}$ by $\|\cdot\|_*$ in the proof above.
\end{proof}

\section{Proof of Example~\ref{ex4} (Kernelization)}
\label{proof:ex4}

\begin{proof}
We assume $\mathbb{H}$ to be an Hilbert space with an inner product operator $\langle\cdot,\cdot\rangle$. The mapping $\phi(\cdot)$ is continuous from $X$ to $\mathbb{H}$. The norm $\|\cdot\|_{\mathbb{H}}:\mathbb{H}\rightarrow \mathbb{R}$ is defined as $\|w\|_{\mathbb{H}}=\sqrt{\langle w,w \rangle}$ for all $w\in \mathbb{H}$, for matrices $\|\mathbf{M}\|_{\mathbb{H}}$ we take the entry wise norm by considering a matrix as a vector, corresponding to the Frobenius norm. The kernel function is defined as $k(x_1,x_2)=\langle\phi(x_1),\phi(x_2)\rangle$. 

$B_\gamma$ and $f_{\mathbb{H}}(\gamma)$ are finite by the compactness of $X$ and continuity of $k(\cdot,\cdot)$. Let $\mathbf{M^*}$ be the solution given training data $p_{\mathbf{s}}$, by the optimality of $\mathbf{M^*}$ and using the same trick as the other examples we have: $\|\mathbf{M^*}\|_{\mathbb{H}} \leq g_0/c$. 
Then, by considering a partition of $\mathcal{Z}$ into $|Y|\mathcal{N}(\gamma/2,X,\|\cdot\|_2)$ disjoint subsets such that if $(x_1,y_1)$ and $(x_2,y_2)$ belong to the same set then $y_1=y_2$ and $\|x_1-x_2\|_2\leq \gamma$. 
We have then, 
\begin{eqnarray}
\lefteqn{|g(y_{ij}[1-f(\mathbf{M^*},\phi(x_1),\phi(x_2))]) - g(y_{ij}[1-f(\mathbf{M^*},\phi(x'_1),\phi(x'_2))])|}\nonumber\\
& \leq & U( |(\phi(x_1)-\phi(x_2))^T\mathbf{M^*}(\phi(x_1)-\phi(x'_1))|
+\nonumber\\
& & |(\phi(x_1)-\phi(x_2))^T\mathbf{M^*}(\phi(x'_2)-\phi(x_2))|+\nonumber\\
& & |(\phi(x_1)-\phi(x'_1))^T\mathbf{M^*}(\phi(x'_1)+\phi(x'_2))| +\nonumber\\
& &|(\phi(x'_2)-\phi(x_2))^T\mathbf{M^*}(\phi(x'_1)+\phi(x'_2))|)\nonumber\\
&\leq&U(|\phi(x_1)^T\mathbf{M^*}(\phi(x_1)-\phi(x'_1))| +|\phi(x_2)^T\mathbf{M^*}(\phi(x_1)-\phi(x'_1))|+\nonumber\\
&&|\phi(x_1)^T\mathbf{M^*}(\phi(x'_2)\phi(x_2))| +|\phi(x_2)^T\mathbf{M^*}(\phi(x'_2)-\phi(x_2))|+\nonumber\\
&&|(\phi(x_1)-\phi(x'_1))^T\mathbf{M^*}\phi(x'_1)| +|(\phi(x_1)-\phi(x'_1))^T\mathbf{M^*}\phi(x'_2)|+\nonumber\\
&&|(\phi(x'_2)-\phi(x_2))^T\mathbf{M^*}\phi(x'_1)| +|(\phi(x'_2)-\phi(x_2))^T\mathbf{M^*}\phi(x'_2)|).\label{eq:l}
\end{eqnarray}

Then, note that 
\begin{eqnarray*}
\lefteqn{|\phi(x_1)^T\mathbf{M^*}(\phi(x_1)-\phi(x'_1))|}\\
&\leq &
\sqrt{\langle\phi(x_1),\phi(x_1)\rangle}
\|\mathbf{M}^*\|_{\mathbb{H}}\sqrt{\langle\phi(x'_1)-\phi(x'_2),\phi(x'_1)-\phi(x'_2)\rangle}\\
&\leq& B_\gamma\frac{g_o}{c}\sqrt{f_{\mathbb{H}}(\gamma)}.
\end{eqnarray*}
Thus, by applying the same principle to all the terms in the right part of  inequality \eqref{eq:l}, we obtain:
\begin{small}
$$\
|g(y_{ij}[1-f(\mathbf{M^*},\phi(x_1),\phi(x_2))]) -
g(y_{ij}[1-f(\mathbf{M^*},\phi(x'_1),\phi(x'_2))])|\leq  \frac{8U B_\gamma \sqrt{f_{\mathbb{H}}(\gamma)}  g_0}{c}.
$$
\end{small}

\end{proof}

\section{Proof of Example~\ref{ex5}}
\label{proof:ex5}

\begin{proof}
Let $\mathbf{M^*}$ be the solution given training data
$p_{\mathbf{s}}$, by the optimality of $\mathbf{M^*}$, we get
$\|\mathbf{M^*}\| \leq g_0/c$ and we consider the same partition of
$\mathcal{Z}$ as in the proof of Example~\ref{ex1}. We can then obtain easily:
\begin{eqnarray*}
\lefteqn{|g(y_{12}[1-f(\mathbf{M^*},x_1,x_2)]) - g(y'_{12}[1-f(\mathbf{M^*},x'_1,x'_2)])|}\\
& \leq & U|x'_1\mathbf{M^*}x'_2-x_1\mathbf{M^*}x_2|\\
& \leq & U|x'_1\mathbf{M^*}x'_2-x_1\mathbf{M^*}x'_2|+U|x_1\mathbf{M^*}x'_2-x_1\mathbf{M^*}x_2|\\
&\leq& U(\|x'_1-x_1\|_2\|\mathbf{M^*}\|_{\mathcal{F}}\|x'_2\|_2 +
\|x_1\|_2\|\mathbf{M^*}\|_{\mathcal{F}}\|x'_2-x_2\|_2)\leq \frac{2UR\gamma g_0}{c}.
\end{eqnarray*}
The proof is given for the Frobenius norm but can be easily adapted
to the use of $\ell_{1}$ and $\ell_{2,1}$ norms using similar arguments as
in the proofs of \ref{proof:ex2} and \ref{proof:ex3}.
\end{proof}

\section{Proof of Example~\ref{ex6}}
\label{proof:ex6}

\begin{proof}
We consider the following loss:
\begin{eqnarray*}
\lefteqn{g([1-(x_i-x_k)^T\mathbf{M}(x_i-x_k)+(x_i-x_j)^T\mathbf{M}(x_i-x_j)])}\\
&=&[1-(x_i-x_k)^T\mathbf{M}(x_i-x_k)+(x_i-x_j)^T\mathbf{M}(x_i-x_j)]_+.
\end{eqnarray*}
Let $\mathbf{M^*}$ be the solution given the training data triplets
$trip_{\mathbf{s}}$. By optimality of $\mathbf{M^*}$,
using the same derivations as above, we
get $\|\mathbf{M^*}\| \leq g_0/c$. 
Then, by considering a partition of $\mathcal{Z}$ into
$|Y|\mathcal{N}(\gamma/2,X,\|\cdot\|_2)$, three partitions $C_1$,
$C_2$, $C_3$ and
$z_1,z_2,z_3,z'_1,z'_2,z'_3\in\mathcal{Z}$ such that $z_1,z'_1\in
C_1$, 
$z_2,z'_2\in C_2$ and  $z_3,z'_3\in C_3$ with $y_1=y'_1=y_2=y'_2$,
$y_3=y'_3$,
$y_3\neq y_1$, and $\|x_1-x'_1\|_1 \leq \gamma$, $\|x_2-x'_2\|_1 \leq
\gamma$, $\|x_3-x'_3\|_1 \leq \gamma$, we have:
\begin{eqnarray*}
\lefteqn{|g([1-(x_1-x_3)^T\mathbf{M^*}(x_1-x_3)+(x_1-x_2)^T\mathbf{M^*}(x_1-x_2)])
  -}\\
\lefteqn{\hspace{.5cm}
  g([1-(x'_1-x'_3)^T\mathbf{M^*}(x'_1-x'_3)+(x'_1-x'_2)^T\mathbf{M^*}(x'_1-x'_2)])|}\nonumber\\
&\leq&
U|(x'_1-x'_3)^T\mathbf{M^*}(x'_1-x'_3)-(x_1-x_3)^T\mathbf{M^*}(x_1-x_3) +\\
&&(x_1-x_2)^T\mathbf{M^*}(x_1-x_2)-(x'_1-x'_2)^T\mathbf{M^*}(x'_1-x'_2)|\\
&\leq&U|(x'_1-x'_3)^T\mathbf{M^*}(x'_1-x'_3)-(x_1-x_3)^T\mathbf{M^*}(x_1-x_3)|+\\
&&U|(x_1-x_2)^T\mathbf{M^*}(x_1-x_2)-(x'_1-x'_2)^T\mathbf{M^*}(x'_1-x'_2)|\\
&\leq&\frac{8UR\gamma g_0}{c}+\frac{8UR\gamma g_0}{c}=\frac{16UR\gamma g_0}{c}.
\end{eqnarray*}
The first inequality is due to the $U$-lipschitz property of $g$, the
second comes from the triangle inequality and the last one follows the
same construction as in the proof of Example~\ref{ex1}. 
Then, by Theorem~\ref{testtheoremtrip}, the example holds.
\end{proof}

\bibliographystyle{unsrt}
\bibliography{refs}

\end{document}